\documentclass{article}

\PassOptionsToPackage{authoryear}{natbib}
\usepackage[preprint]{neurips_2026}

\setcitestyle{authoryear,round}
\usepackage{times}

\usepackage{amsmath,amsfonts,amssymb,amsthm}

\usepackage{graphicx}
\usepackage{booktabs}
\usepackage{subcaption}

\usepackage{microtype}
\usepackage{hyperref}
\usepackage{xcolor}
\usepackage{enumitem}
\usepackage{epigraph}
\newcommand{\draft}[1]{}

\newtheorem{theorem}{Theorem}

\title{The Curse and Blessing of Mean Bias in FP4-Quantized LLM Training}

\author{%
\textbf{Hengjie Cao}$^{1,3}$\thanks{Co-first authors.} \hspace{1.0em}
\textbf{Zhendong Huang}$^{1}$\footnotemark[1] \hspace{1.0em}
\textbf{Mengyi Chen}$^{1}$ \hspace{1.0em}
\textbf{Yifeng Yang}$^{1}$ \hspace{1.0em}
\textbf{Fang Dong}$^{1}$ \hspace{1.0em}
\\
\textbf{Anrui Chen}$^{1}$ \hspace{1.0em}
\textbf{Ruijun Huang}$^{1}$ \hspace{1.0em}
\textbf{Xin Zhang}$^{1}$ \hspace{1.0em}
\textbf{Mingzhi Dong}$^{2}$ \hspace{1.0em}
\textbf{Yujiang Wang}$^{4,5}$ \hspace{1.0em}
\\
\textbf{Jinlong Hou}$^{3}$ \hspace{1.0em}
\textbf{Qin Lv}$^{6}$ \hspace{1.0em}
\textbf{Robert P. Dick}$^{7}$ \hspace{1.0em}
\textbf{Yuan Cheng}$^{3}$ \hspace{1.0em}
\textbf{Tun Lu}$^{1}$ \hspace{1.0em} 
\\
\textbf{Fan Yang}$^{1}$ \hspace{1.0em}
\textbf{Yixuan Chen}$^{4}$ \hspace{1.0em}
\textbf{Li Shang}$^{1,8}$ \hspace{1.0em}
\\[0.5em]
$^1$ Fudan University \hspace{1.0em}
$^2$ University of Bath \hspace{1.0em}
$^3$ Shanghai Innovation Institute \hspace{1.0em}\\
$^4$ University of Oxford \hspace{1.0em}
$^5$ Oxford Suzhou Centre for Advanced Research \hspace{1.0em}  \\
$^6$ University of Colorado Boulder \hspace{1.0em}
$^7$ University of Michigan \hspace{1.0em}
$^8$ Shenzhen Loop Area Institute \hspace{1.0em}
}

\begin{document}

\maketitle

\begin{abstract}
FP4 training promises substantial memory and compute savings for large language models, but remains fragile because blockwise quantization is dictated by extreme activation magnitudes, which inflate dynamic range and compress long-tail signals. We identify a counterintuitive source of this failure: dominant activation outliers are not merely arbitrary sparse events, but are largely induced by a coherent rank-one mean bias, whose direction aligns with the leading anisotropic spectral component.
This mean component strengthens during training, is amplified and reshaped by attention and FFN operators, and increasingly dominates top activation magnitudes.
Crucially, this discovery reveals that a seemingly complex outlier-suppression problem admits a truly simple solution: isolate the coherent mean before quantization. We therefore propose Averis, a mean–residual splitting quantization method that separates the mean component using only reductions and elementwise subtractions before FP4 quantization.
Across Qwen3-0.6B Dense trained on 100B tokens and Qwen3-7B-A1.5B MoE trained on 50B tokens, Averis enables robust W4A4G4 FP4 training, reducing BF16 loss gaps to 1.19\%/0.81\% versus 2.05\%/1.10\% for NVIDIA's recently released Hadamard-based outlier-smoothing method, while limiting downstream gaps to 0.89/0.71 points. With only 2.20\% end-to-end overhead over vanilla NVFP4, about 30\% of NVIDIA's Hadamard-based design, Averis provides a hardware-efficient path to stable low-bit LLM training. Complementary to Hadamard, Averis further reduces the Qwen3-0.6B loss and downstream gaps to 0.94\% and 0.73 points when combined. Code is available at: https://anonymous.4open.science/r/averis-504D.

\end{abstract}


\section{Introduction}
Training large language models (LLMs) with low-bit quantization of parameters, activations, and gradients offers substantial gains in memory footprint, and computational efficiency. In the FP4 regime, NVIDIA's NVFP4 format reduces memory consumption by 1.8$\times$ and accelerates General Matrix Multiplications (GeMMs) by 7$\times$ compared with FP8~\citep{alvarez2025nvfp4_inference, devleker2025nvfp4}. In our end-to-end Qwen3-8B~\citep{yang2025qwen3} training profile on B300 GPUs, FP8 and NVFP4 achieve $1.38\times$ and $1.61\times$ per-step speedups, respectively, over the BF16 profile. However, pushing the training frontier down to FP4 often comes with training-loss gaps and downstream performance degradation~\citep{wang2025fp4training,chmiel2025fp4all,tseng2025mxfp4,cao2025metis,ashkboos2024quarot,li2024svdquant}.

The central challenge in FP4 training is the presence of extreme outlier magnitudes in training tensors, especially activation matrices that participate in both forward and backward GeMMs~\citep{ashkboos2024quarot,ashkboos2025halo}. Under blockwise quantization, a few activation outliers can set the block scale, inflate the dynamic range, and compress the long-tail signal of ordinary entries into narrow numerical bins, degrading both forward representation fidelity and backward gradient propagation~\citep{xiao2023smoothquant,ashkboos2024quarot,li2024svdquant,cao2025metis}. Existing mitigation strategies mainly fall into two categories. One operates in the element space, typically using tiled Hadamard transforms to smooth outliers within each tile \citep{ashkboos2024quarot, ashkboos2025halo}, but with limited ability to contract the overall dynamic range and eliminate dominant outliers. 
The other operates in the spectral space, relying on explicit spectral control such as singular value decomposition (SVD), as in Metis~\citep{cao2025metis}, which typically achieves lower quantization error but remains computationally intensive and poorly aligned with modern accelerator hardware.


We identify a counterintuitive structural source of FP4 training instability: dominant activation outliers are not merely arbitrary sparse events, but are largely induced by a coherent rank-one mean bias. This bias corresponds to the feature-wise mean vector of the activation matrix, forms a shared offset across tokens, and aligns with the leading spectral component. Empirically, the mean component strengthens as training progresses while maintaining this spectral alignment. Operator-level analysis further shows that attention and FFN modules actively amplify and reshape the coherent mean component.



We further show that this rank-one mean bias is the dominant source of extreme activation magnitudes. Empirically, in deep layers the top-0.1\% activation entries shift from residual-dominated at early training to mean-dominated at late training, with their median squared mean-share surging from initially negligible levels to approximately 95\%. Theoretically, under an empirically validated Gaussian residual model, a columnwise mean that is large relative to the residual variance exponentially amplifies far‑tail exceedance probabilities relative to zero‑mean fluctuations, providing a mathematical account of how the largest activations become mean-driven as the bias strengthens.




Crucially, this discovery reveals that a seemingly complex outlier-suppression problem admits a truly simple solution: isolate the coherent mean before quantization. We therefore propose \emph{Averis} (Averaging-Induced Residual Splitting), which explicitly factors activations and output gradients into column-mean and residual components, and quantizes them separately in both forward and backward GeMMs. By removing the dominant bias-induced dynamic range from the residual stream, Averis reduces quantization error using only mean reductions and elementwise subtractions.

Averis enables robust W4A4G4 training by quantizing all GeMM matrices to FP4. Across Qwen3-0.6B Dense trained on 100B tokens and Qwen3-7B-A1.5B MoE trained on 50B tokens, Averis reduces the BF16 loss gaps to 1.19\% and 0.81\%, compared with 2.05\% and 1.10\% for NVIDIA's recently released Hadamard-based outlier-smoothing method. The gains also translate to downstream performance, where Averis limits the accuracy gaps to 0.89 and 0.71 points under NVFP4 forward evaluation. On NVIDIA Blackwell GPUs, Averis incurs only 2.20\% end-to-end overhead over vanilla NVFP4, about 30\% of the Hadamard-based design, and delivers up to $1.57\times$ end-to-end speedup over BF16. Moreover, Averis is complementary to Hadamard transformation: on Qwen3-0.6B, their combination further reduces the loss and downstream gaps to 0.94\% and 0.73 points. Overall, Averis demonstrates a practical and hardware-efficient path to stable FP4 LLM training.

\section{Analysis}

\subsection{Mean Bias Phenomenon}
\label{sec:mean_bias}

\textbf{Notation.}
Let $\mathcal{X}\in\mathbb{R}^{b\times s\times m}$ denote an activation tensor, where
$b$ is the batch size, $s$ is the sequence length, and $m$ is the hidden dimension.
For GeMM-level analysis, we reshape $\mathcal{X}$ into a matrix
$\mathbf{X}\in\mathbb{R}^{l\times m}$ with $l=bs$.
We write the singular value decomposition of $\mathbf{X}$ as
$\mathbf{X}=\sum_{k=1}^{r}\sigma_k \mathbf{u}_k\mathbf{v}_k^\top$,
where $r=\min(l,m)$ and $\sigma_1\ge \sigma_2\ge \cdots \ge \sigma_r>0$.
We define the feature-wise mean vector as
$\boldsymbol{\mu}_{\mathbf X}\triangleq \frac{1}{l}\mathbf{X}^\top \mathbf{1}\in\mathbb{R}^m$, where $\mathbf{1}\in\mathbb{R}^l$ denotes the all-ones vector,
the corresponding rank-one mean matrix as
$\mathbf{M}_{\mathbf X}\triangleq \mathbf{1}\boldsymbol{\mu}_{\mathbf X}^\top\in\mathbb{R}^{l\times m}$,
and the centered activation matrix as
$\widetilde{\mathbf X}\triangleq \mathbf X-\mathbf M_{\mathbf X}$.
We also define the normalized all-ones direction
$\mathbf e\triangleq \mathbf 1/\sqrt l\in\mathbb{R}^l$
and the alignment coefficients
$\beta_k\triangleq \langle \mathbf u_k,\mathbf e\rangle$.

\textbf{Mean bias: a common directional shift in representation space.}
We define \emph{mean bias} as a coherent shift of token representations toward a shared direction. Figure~\ref{fig:mean_bias_definition}(B) illustrates this phenomenon using the activations from the deepest layer of Qwen3-0.6B dense model at late stage: token-wise cosine similarities with the mean direction are almost uniformly positive, whereas those with a non-mean direction, chosen as the second right singular vector $\mathbf v_2$, alternate in sign. This contrast indicates that the mean direction captures a shared representation-level offset rather than ordinary variance. The same pattern holds across additional layers and training stages, as shown in Appendix~\ref{app:mean_bias_more_layers}.

\textbf{Mean bias aligns with the dominant anisotropic spectral direction.}
We further find that this coherent mean direction is tightly coupled to the leading spectral spike. Substituting the SVD of $\mathbf X$ into the definition of $\boldsymbol{\mu}_{\mathbf X}$ gives
$\boldsymbol{\mu}_{\mathbf X}
=
\frac{1}{\sqrt l}\sum_{k=1}^{r}\sigma_k \beta_k \mathbf v_k .$
This expansion shows that the mean direction is controlled jointly by spectral anisotropy, through the singular values $\sigma_k$, and by token-wise coherence, through the alignment coefficients $\beta_k$. In LLM activations, the spectrum is strongly anisotropic, with a dominant leading singular value as shown in Figure~\ref{fig:mean_bias_definition}(A). When this leading mode also has large alignment with the all-ones token direction, while non-leading modes exhibit stronger sign cancellation, the mean vector concentrates along $\mathbf v_1$. Consistently, Figure~\ref{fig:mean_bias_definition}(C) shows that $\boldsymbol{\mu}_{\mathbf X}$ has much larger cosine similarity with $\mathbf v_1$, approaches $0.99$, than with other right singular directions. 

\begin{figure}[h]
\centering
\includegraphics[width=\linewidth]{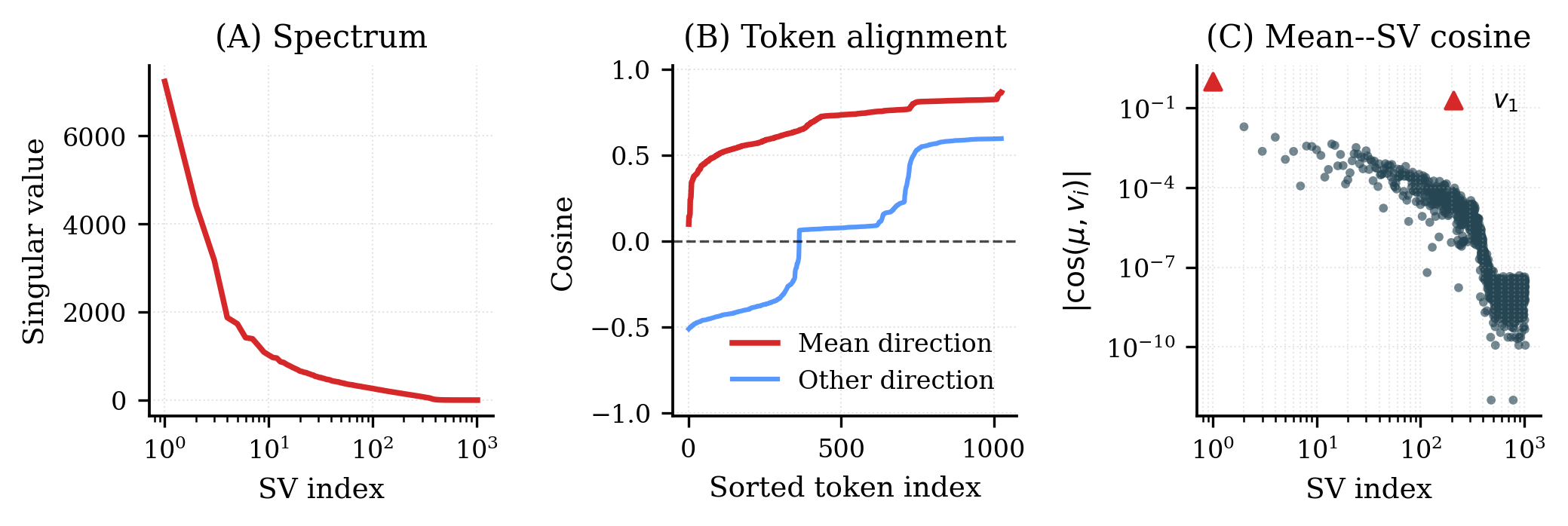}
\caption{Mean-bias evidence in layer-27 FFN-input activations of Qwen3-0.6B at step 170k. (A) A dominant leading singular value indicates strong anisotropy. (B) Token-wise cosine similarities stay mostly positive along the mean direction but are weaker and more mixed along a non-leading direction. (C) The mean vector aligns most strongly with the top right singular vector.}
\label{fig:mean_bias_definition}
\end{figure}

\subsection{Structural Emergence of Mean Bias}
\label{sec:mean_bias_emergence}

We examine the evolution of \emph{mean bias} across training. Globally, the feature-wise mean strengthens over training while staying aligned with the leading anisotropic direction; locally, attention and FFN modules amplify and reshape this coherent component within Transformer blocks.

\textbf{Empirical Trajectory: Mean Bias Strengthens During Training.}
We quantify the strength of mean bias by the normalized ratio
$R \triangleq \frac{\|\boldsymbol{\mu}_{\mathbf X}\|_2}{\sqrt{\|\mathbf X\|_F^2/l}}$.
Using activation matrices from Qwen3-0.6B at different training stages, we track $R$ and the cosine alignment between the mean vector and the top right singular vector across layers.
As shown in Figure~\ref{fig:energy_grid}, the mean-bias energy generally increases as training progresses, most prominently in early and middle layers.
Meanwhile, the mean vector remains strongly aligned with the top singular direction across layers and stages, often approaching cosine similarity $1.0$.
These trends show that mean bias is progressively strengthened during training while remaining a structured rank-one component coupled to the dominant anisotropic direction.

\begin{figure*}[h]
\centering
\includegraphics[width=\linewidth]{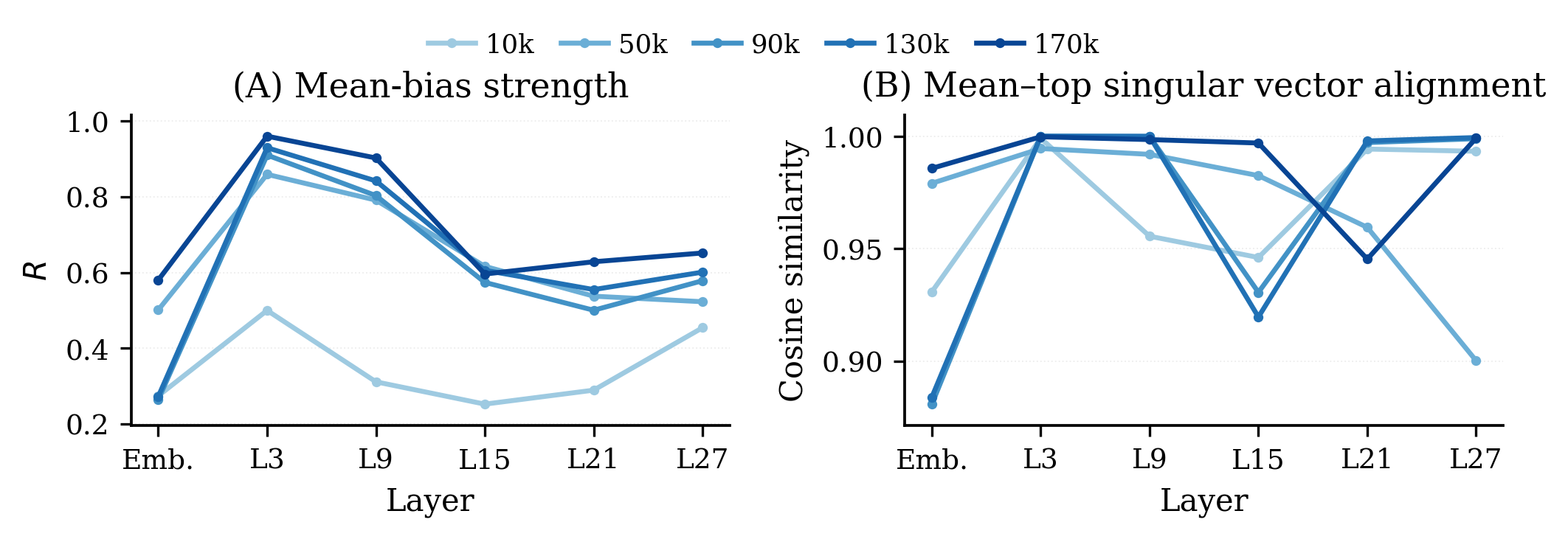}
\caption{Evolution of mean bias across depth and training in Qwen3-0.6B activations. Left: the normalized mean-bias ratio $R$ increases with training step. Right: the growing mean component remains tightly aligned with the dominant anisotropic direction.}
\label{fig:energy_grid}
\end{figure*}


 
\textbf{Operator-Level Regeneration and Amplification.}
Transformer operators actively amplify mean bias, while also reshaping its direction across training stages. We verify this by tracing the mean-bias ratio $R$ across operator stages and measuring adjacent-stage mean-vector cosine. Figure~\ref{fig:operator_change}(A) shows that, at the early checkpoint, both attention and FFN increase $R$, with the FFN also substantially redirecting the mean direction. At the late checkpoint, Figure~\ref{fig:operator_change}(B) shows that amplification is mainly concentrated in the FFN path, while attention contributes more to directional reshaping. 

\begin{figure}[h]
\centering
\includegraphics[width=\linewidth]{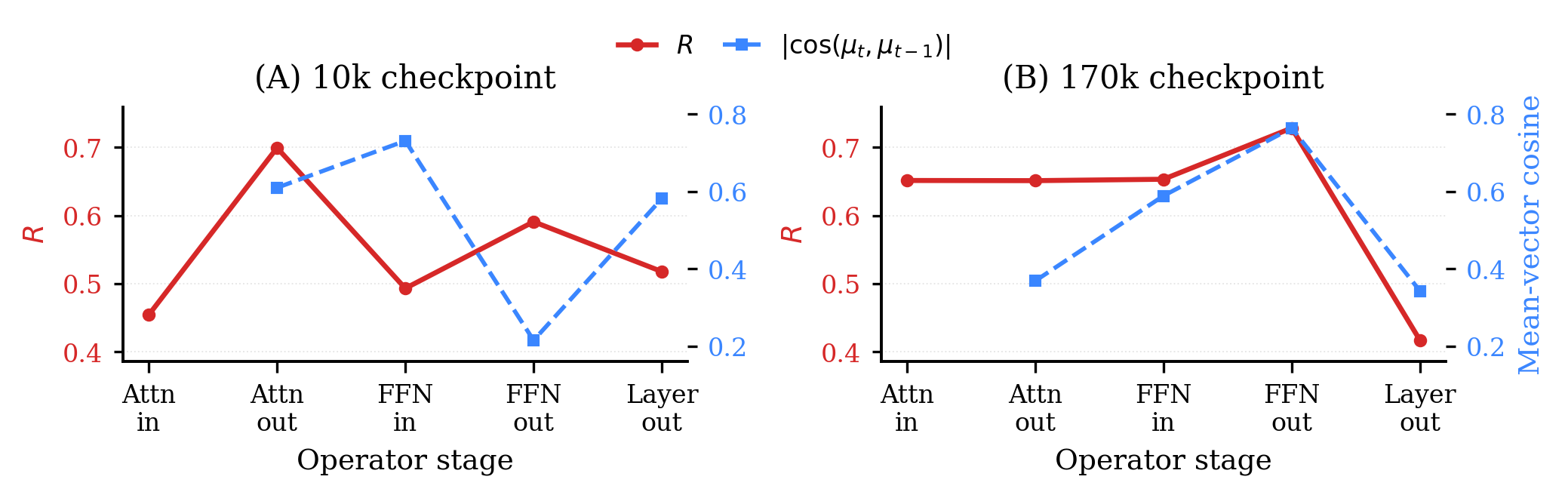}
\caption{Operator-level evidence for mean-bias amplification. We compare the evolution of mean-bias strength and mean-direction alignment across key operators, showing that nonlinear Transformer components further strengthen the coherent mean component during forward propagation.}
\label{fig:operator_change}
\end{figure}


\subsection{Mean Bias as the Dominant Source of Activation Outliers}
\label{sec:mean_outliers}



We explain the sharp numerical consequence of mean bias: it produces extreme elementwise activations that set blockwise quantization scales, inflate dynamic range, and compress ordinary long-tail signals into narrow numerical bins.


\textbf{Empirical outlier attribution.}
To localize the source of extreme activation values, let $\mathcal E_{\mathrm{top}}$
denote the index set of the top-$0.1\%$ entries of $\mathbf X$ ranked by absolute
magnitude. For each $(i,j)\in\mathcal E_{\mathrm{top}}$, we measure the
component-wise squared contribution proxies $\rho_{ij}^{(\mathrm{mean})}
\triangleq
\frac{(\mathbf M_{\mathbf X})_{ij}^{2}}{X_{ij}^{2}},
\qquad
\rho_{ij}^{(\mathrm{res})}
\triangleq
\frac{\widetilde X_{ij}^{2}}{X_{ij}^{2}},$
which serve as empirical indicators of whether extreme entries are dominated by the coherent mean component or by the residual.
Figure~\ref{fig:outlier_shares} reports this attribution for activations of Qwen3-0.6B at early and late checkpoints, comparing a shallow
layer with a deep layer. At the early checkpoint, top outliers are largely
residual-dominated. As training progresses, the mean contribution shifts toward larger values, where late-stage outliers exhibit a substantial concentration of high mean-share values. 

\begin{figure}[h]
\centering
\includegraphics[width=\linewidth]{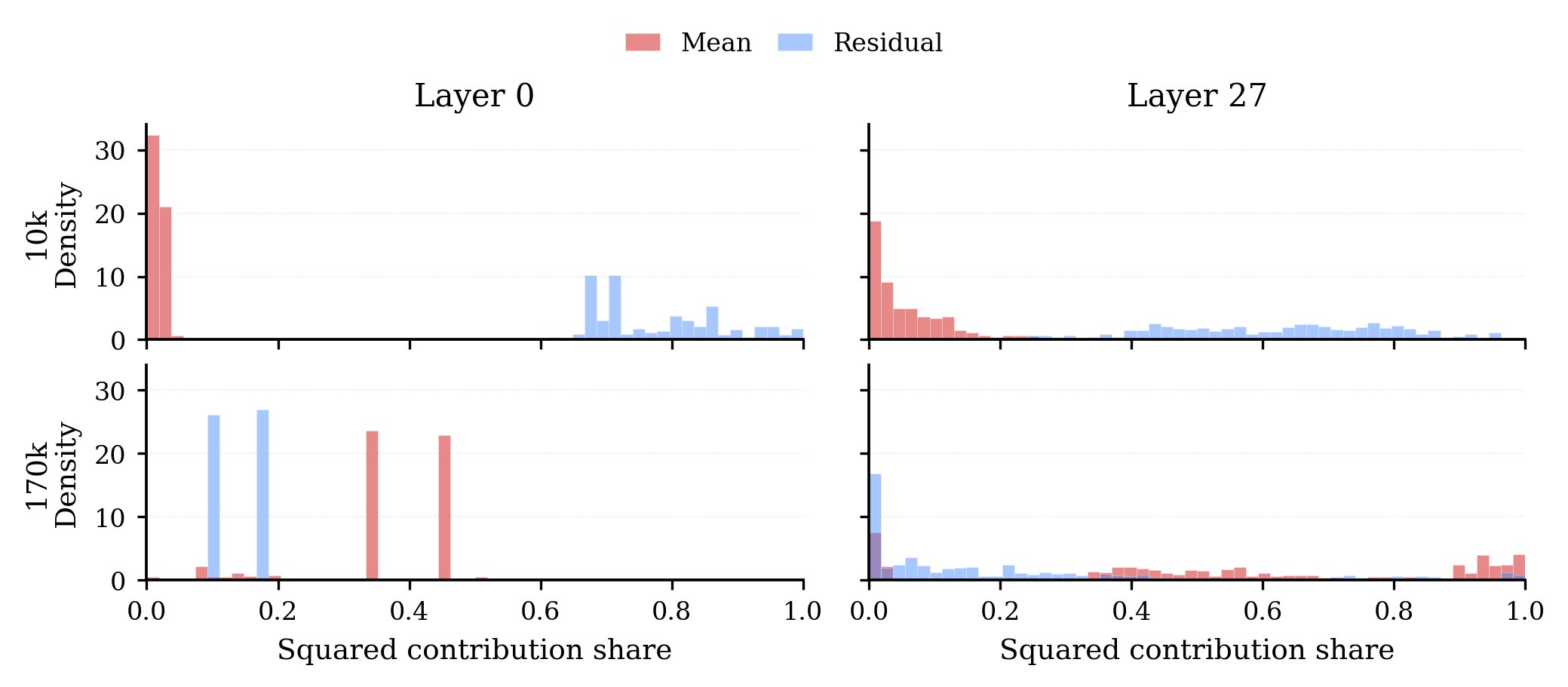}
\caption{Histograms of activation outlier mean and residual contribution shares (top 0.1\% entries) across layers and training steps.}
\label{fig:outlier_shares}
\end{figure}

\textbf{Theoretical model: mean-bias amplification of columnwise outliers.}
We now formalize this effect at the level of a single feature coordinate. For a
fixed column $j$, let
\begin{equation}
m_j \triangleq (\boldsymbol{\mu}_{\mathbf X})_j
=
\frac{1}{\sqrt l}\sum_{k=1}^{r}\sigma_k\beta_k v_{kj},
\label{eq:full_column_mean}
\end{equation}
We also denote the leading-mode contribution to this column mean by
$m_j^{(1)}
\triangleq
\frac{\sigma_1\beta_1}{\sqrt l}v_{1j}.
\label{eq:top_mode_column_mean}$

Let $I\sim\mathrm{Unif}(\{1,\dots,l\})$ and define the row-sampled entry
\[
Y_j \triangleq X_{Ij}=m_j+\eta_{Ij},
\qquad
\eta_{Ij}\triangleq X_{Ij}-m_j .
\]
Thus $\eta_{Ij}$ is the centered residual obtained by sampling one token position
from column $j$.

We use the following empirically verified assumptions.

\smallskip
\noindent\textbf{Assumption 1: Marginal Gaussian residual.}
For each column $j$, the row-sampled centered residual is approximated by a
Gaussian marginal distribution,
\begin{equation}
\eta_{Ij}\ \dot\sim\ N(0,\tau_j^2).
\label{eq:assumption_gaussian_residual}
\end{equation}
This assumption concerns only the marginal distribution of a randomly sampled
coordinate and does not require the residuals
$\{\eta_{ij}\}_{i=1}^{l}$ to be mutually independent. Empirically, Figure~\ref{fig:gaussian_residual}
shows that the raw activation distribution deviates strongly from a Gaussian fit,
whereas the residual after mean removal is substantially closer to Gaussian.

\begin{figure}[h]
\centering
\includegraphics[width=\linewidth]{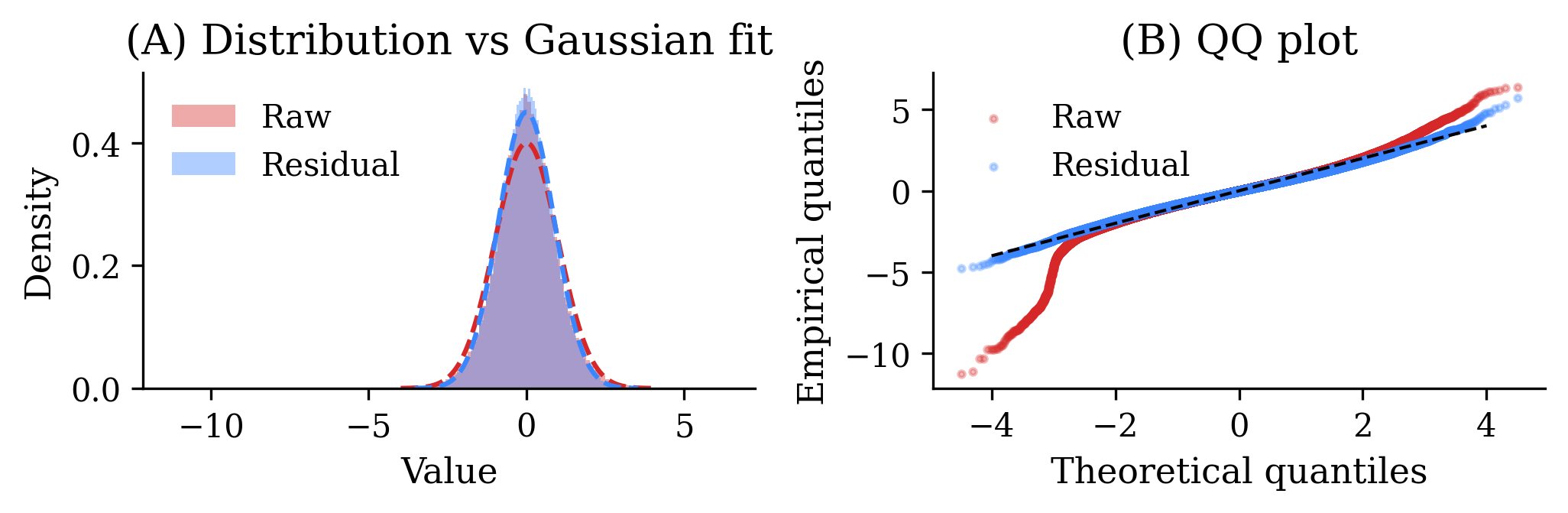}
\caption{Empirical validation of the marginal Gaussian residual assumption.
Raw activations deviate from Gaussianity, while residuals are much
closer to a Gaussian fit in both density and QQ plots.}
\label{fig:gaussian_residual}
\end{figure}

\smallskip
\noindent\textbf{Assumption 2: Diagonal variance approximation.}
The residual variance of column $j$ is well approximated by its diagonal spectral
contribution:
\begin{equation}
\tau_j^2
\approx
\tau_{j,\mathrm{diag}}^2
\triangleq
\frac{1}{l}
\sum_{k=1}^{r}
\sigma_k^2(1-\beta_k^2)v_{kj}^2 .
\label{eq:assumption_diagonal_variance}
\end{equation}
Equivalently, cross-mode covariance terms contribute little to the row-sampling
variance. This approximation is empirically validated in Appendix~\ref{app:variance_decomp},
which shows that the diagonal estimate closely tracks the empirical residual
variance and that the relative cross-term contribution remains small.

\smallskip
\noindent\textbf{Assumption 3: Mean-dominant aligned columns.}
For columns that contribute to large outliers, the coherent mean dominates the
residual scale, $\frac{|m_j|}{\tau_j}\gg 1,
\label{eq:assumption_mean_dominance}$
and the column mean is mainly induced by the leading aligned spectral mode, $|m_j-m_j^{(1)}|\ll |m_j^{(1)}|.
\label{eq:assumption_aligned_spike}$
This condition is favored by the empirical structure established in
Section~\ref{sec:mean_bias}: the leading singular value is dominant, the leading
left singular vector is strongly aligned with the all-ones token direction, and the mean vector is concentrated along the top right singular vector.

\smallskip
\begin{theorem}[Mean-bias amplification of columnwise outliers]
\label{thm:mean_bias_tail}
Fix a column $j$ and let $Y_j=X_{Ij}=m_j+\eta_{Ij}$ with
$I\sim\mathrm{Unif}(\{1,\dots,l\})$. Suppose Assumption~1 holds exactly with
$\eta_{Ij}\sim N(0,\tau_j^2)$. Then, for any threshold $t>0$,
\begin{equation}
\mathbb P(|Y_j|>t)
=
Q\!\left(\frac{t-|m_j|}{\tau_j}\right)
+
Q\!\left(\frac{t+|m_j|}{\tau_j}\right),
\label{eq:tail_formula}
\end{equation}
where $Q(x)=1-\Phi(x)$ and $\Phi$ is the standard normal CDF.

Let $Y_j^{(0)}\sim N(0,\tau_j^2)$ be the corresponding zero-mean variance-only
baseline. In the high-threshold regime
\begin{equation}
t>|m_j|,
\qquad
\frac{t-|m_j|}{\tau_j}\to\infty,
\qquad
\frac{t|m_j|}{\tau_j^2}\to\infty,
\label{eq:far_tail_conditions}
\end{equation}
the tail probability satisfies
\begin{equation}
\mathbb P(|Y_j|>t)
\sim
Q\!\left(\frac{t-|m_j|}{\tau_j}\right),
\label{eq:tail_asymptotic}
\end{equation}
and its amplification relative to the variance-only baseline is
\begin{equation}
\frac{\mathbb P(|Y_j|>t)}
{\mathbb P(|Y_j^{(0)}|>t)}
\sim
\frac{t}{2(t-|m_j|)}
\exp\!\left(
\frac{2t|m_j|-m_j^2}{2\tau_j^2}
\right).
\label{eq:tail_ratio}
\end{equation}

\end{theorem}

\noindent\textit{Proof.} Deferred to Appendix~\ref{appendix:proof_mean_bias_tail}.


Consequently, when $|m_j|/\tau_j$ is large, high-threshold exceedances are exponentially amplified by the coherent mean shift. If Assumption~3 also holds, this amplification is controlled by the rank-one mean-bias component induced by the leading aligned mode,
$\frac{\sigma_1\beta_1}{\sqrt l}\mathbf 1\mathbf v_1^\top$.
This theorem provides a concrete mechanism for the empirical outlier attribution observed above: after mean removal, the remaining residual fluctuations are well modeled by a Gaussian marginal, while a nonzero column mean shifts the tail by $|m_j|$ and thereby increases high-threshold exceedance probabilities exponentially relative to a zero-mean variance-matched baseline. 

We further verify this implication empirically in Appendix~\ref{app:residual_outlier_distribution}.
For late-stage activations in both shallow and deep layers, subtracting the feature-wise mean substantially contracts the high-magnitude tail of the value distribution. 


\section{Method}
\label{sec:lowbit_method}
Building on the insight above, we identify a simple design principle for low-bit training: since quantization-sensitive outliers are predominantly driven by a coherent rank-one mean component, it is sufficient to isolate this component before quantization. We propose \emph{Averis} (Averaging-Induced Residual Splitting), a mean--residual quantization method that explicitly separates the column-mean component from the residual and quantizes them independently. Averis applies this decomposition to activations and output gradients in the forward and backward GeMMs; although output gradients show weaker mean bias, mean centering still slightly reduces quantization error, as verified in Appendix~\ref{app:gradout_mean_centering}. It requires only mean reductions,  remaining well aligned with efficient accelerator execution.

\textbf{Notation and quantization operator.}
Let $\mathcal{Q}_b(\cdot)$ denote a $b$-bit quantization operator (e.g., blockwise FP4) applied to a matrix.
We write $\bar{\mathbf{M}}\triangleq \mathcal{Q}_b(\mathbf{M})$ for the quantized version of any matrix
$\mathbf{M}$.
Throughout, for any positive integer $d$, we write $\mathbf{1}_d\in\mathbb{R}^{d}$ for the all-ones column vector of length $d$.

\textbf{Forward pass: activation mean--residual splitting.}
Let $\mathbf{W}\in\mathbb{R}^{m\times n}$ be a weight matrix and let
$\mathcal{X}\in\mathbb{R}^{b\times s\times m}$ be an input activation tensor, where $b$ is batch size,
$s$ is sequence length, and $m$ is hidden dimension.
As usual, we reshape $\mathcal{X}$ into a matrix $\mathbf{X}\in\mathbb{R}^{l\times m}$ with $l=b\cdot s$
and compute the GeMM output $\mathbf{Y}=\mathbf{X}\mathbf{W}\in\mathbb{R}^{l\times n}.$

We define the \emph{activation column-mean vector} $\mu_{\mathbf{X}}\in\mathbb{R}^{m}$ by $\mu_{\mathbf{X}} \triangleq \frac{1}{l}\mathbf{1}_l^{\top}\mathbf{X},$
The residual is $\mathbf{X}_{\mathrm{R}} \triangleq \mathbf{X}-\mathbf{1}_l\,\mu_{\mathbf{X}}.$
We then quantize the mean vector and residual independently:
\[
\bar{\mu}_{\mathbf{X}}=\mathcal{Q}_b(\mu_{\mathbf{X}}),
\qquad
\bar{\mathbf{X}}_{\mathrm{R}}=\mathcal{Q}_b(\mathbf{X}_{\mathrm{R}}),
\qquad
\bar{\mathbf{W}}=\mathcal{Q}_b(\mathbf{W}).
\]
The quantized forward computation is
\begin{equation}
\label{eq:averis_forward}
\hat{\mathbf{Y}}
\;=\;
\mathbf{1}_l\,(\bar{\mu}_{\mathbf{X}}\bar{\mathbf{W}})
\;+
\bar{\mathbf{X}}_{\mathrm{R}}\,\bar{\mathbf{W}}.
\end{equation}
where the first term is broadcast along the token dimension via $\mathbf{1}_l$.
In practice, one computes $\mu_{\mathbf{X}}$ first, subtracts it from $\mathbf{X}$ on the fly to form
$\mathbf{X}_{\mathrm{R}}$, and never materializes the mean matrix.

\textbf{Backward pass: output-gradient mean--residual splitting.}
Let $\mathbf{D}\in\mathbb{R}^{l\times n}$ denote the matrix of loss derivatives with respect to
the GeMM output $\mathbf{Y}$, i.e., $\mathbf{D}\triangleq \frac{\partial \mathcal{L}}{\partial \mathbf{Y}}.$
We define the \emph{output-gradient column-mean vector} $\mu_{\mathbf{D}}\in\mathbb{R}^{n}$ by
\[
\mu_{\mathbf{D}} \triangleq \frac{1}{l}\mathbf{1}_l^{\top}\mathbf{D},
\qquad
\mathbf{D}_{\mathrm{R}} \triangleq \mathbf{D}-\mathbf{1}_l\,\mu_{\mathbf{D}}.
\]
Quantize the components independently:
\[
\bar{\mu}_{\mathbf{D}}=\mathcal{Q}_b(\mu_{\mathbf{D}}),
\qquad
\bar{\mathbf{D}}_{\mathrm{R}}=\mathcal{Q}_b(\mathbf{D}_{\mathrm{R}}).
\]

For the input-gradient GeMM, Averis evaluates
\begin{equation}
\label{eq:averis_backward_dx}
\widehat{\frac{\partial \mathcal{L}}{\partial \mathbf{X}}}
=
\mathbf{1}_l\,(\bar{\mu}_{\mathbf{D}}\bar{\mathbf{W}}^{\top})
+
\bar{\mathbf{D}}_{\mathrm{R}}\bar{\mathbf{W}}^{\top}.
\end{equation}

For the weight-gradient GeMM, the exact mean--residual expansion gives
\[
\frac{\partial \mathcal{L}}{\partial \mathbf{W}}
=
\mathbf{X}^{\top}\mathbf{D}
=
\mathbf{X}_{\mathrm R}^{\top}\mathbf{D}_{\mathrm R}
+
\mathbf{X}_{\mathrm R}^{\top}(\mathbf{1}_l\mu_{\mathbf D})
+
(\mathbf{1}_l\mu_{\mathbf X})^{\top}\mathbf{D}_{\mathrm R}
+
(\mathbf{1}_l\mu_{\mathbf X})^{\top}(\mathbf{1}_l\mu_{\mathbf D}) .
\]
Since $\mathbf{X}_{\mathrm R}$ and $\mathbf{D}_{\mathrm R}$ are column-centered, $\mathbf{X}_{\mathrm R}^{\top}\mathbf{1}_l=0,
\mathbf{1}_l^\top\mathbf{D}_{\mathrm R}=0,$
and therefore the two cross terms vanish: $\mathbf{X}_{\mathrm R}^{\top}(\mathbf{1}_l\mu_{\mathbf D})=0, (\mathbf{1}_l\mu_{\mathbf X})^{\top}\mathbf{D}_{\mathrm R}=0.$
Thus Averis evaluates the quantized weight-gradient GeMM as
\begin{equation}
\label{eq:averis_backward_dw}
\widehat{\frac{\partial \mathcal{L}}{\partial \mathbf{W}}}
=
\bar{\mathbf{X}}_{\mathrm R}^{\top}\bar{\mathbf{D}}_{\mathrm R}
+
l\,\bar{\mu}_{\mathbf X}^{\top}\bar{\mu}_{\mathbf D}.
\end{equation}

\section{Experiments}


\textbf{Models and Datasets.}
We conduct experiments on two model settings: Qwen3-0.6B Dense and Qwen3-7B-A1.5B MoE. The MoE model is a scaled-down variant following Qwen3-235B-A22B.
For both models, we use the DCLM~\citep{li2024datacomp} dataset for pretraining. The Qwen3-0.6B Dense model is trained on 100B tokens, and the Qwen3-7B-A1.5B MoE model is trained on 50B tokens.

\textbf{FP4 Training.}
All low-bit experiments adopt W4A4G4 quantization, where weights, activations,
and gradients are represented in the E2M1 NVFP4 format. Stochastic rounding (SR) is applied to the backward gradient GeMMs to mitigate quantization bias, and is orthogonal to other methods. Training-quality experiments for both model scales are conducted with FP4 simulation on Hopper GPUs, while the end-to-end runtime overhead is evaluated on Blackwell GPUs.

\textbf{Baselines and Evaluation.}
We compare Averis against full-precision BF16 training, vanilla FP4 training, and NVIDIA's recently released Hadamard-based outlier-smoothing design for FP4 training. We evaluate different methods by the training-loss gap relative to BF16 and downstream task performance. For low-precision settings, downstream evaluation is also performed with NVFP4 quantized forward computation. For downstream evaluation, we consider three task types: question answering (ARC~\cite{clark2018think}, RACE~\cite{lai2017race}, BoolQ~\cite{clark2019boolq}), classification (HellaSwag~\cite{zellers2019hellaswag}, PIQA~\cite{bisk2020piqa}), and cloze prediction (LAMBADA~\cite{kazemi2023lambada}).

\subsection{Main Results}

Figure~\ref{fig:loss_curve} compares the training-loss trajectories of Qwen3-0.6B and Qwen3-7B-A1.5B, and Table~\ref{tab:loss_downstream_combined} reports the final statistics. The behavior differs across model scales. On Qwen3-0.6B, Averis substantially narrows the BF16 loss gap, reducing it from 2.70\% under vanilla NVFP4 and 2.05\% under NVFP4-Hadamard to 1.19\%. On Qwen3-7B-A1.5B, all reported low-bit variants remain close to BF16, with loss gaps of 1.03\% for vanilla NVFP4, 1.10\% for NVFP4-Hadamard, and 0.81\% for Averis. These results show that mean-bias isolation is especially effective at the 0.6B scale, while the larger model exhibits a more nuanced trade-off among low-bit recipes.

\begin{figure*}[h]
\centering
\begin{subfigure}{0.48\textwidth}
    \centering
    \includegraphics[width=\linewidth]{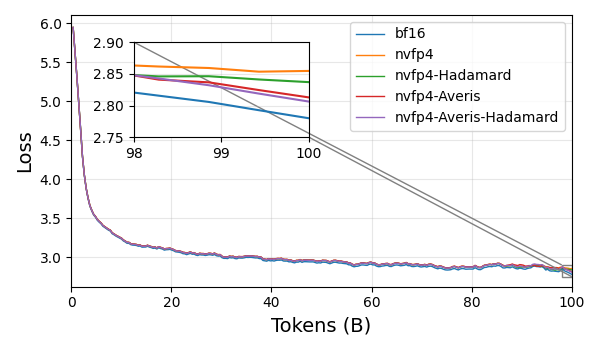}
    \caption{Training-loss curve for Qwen3-0.6B.}
    \label{fig:loss_curve_06b}
\end{subfigure}\hfill
\begin{subfigure}{0.48\textwidth}
    \centering
    \includegraphics[width=\linewidth]{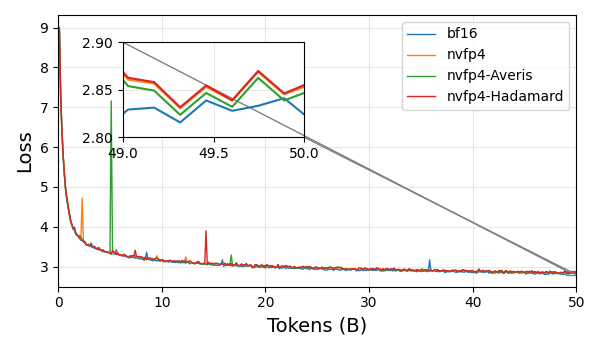}
    \caption{Training-loss curve for Qwen3-7B-A1.5B.}
    \label{fig:loss_curve_7b}
\end{subfigure}
\caption{Training-loss curves across model scales. Left: Qwen3-0.6B. Right: Qwen3-7B-A1.5B. We compare BF16, vanilla NVFP4, NVFP4-Hadamard, Averis, and Averis-Hadamard.}
\label{fig:loss_curve}
\end{figure*}

\textbf{Downstream Performance.}
Table~\ref{tab:loss_downstream_combined} further shows that the convergence behavior is reflected in downstream performance under NVFP4 forward evaluation. On Qwen3-0.6B, vanilla NVFP4 drops the average downstream score by 1.26 points relative to BF16, and NVFP4-Hadamard reduces this gap to 0.98 points. Averis further narrows the gap to 0.89 points, while Averis-Hadamard achieves the smallest gap at 0.73 points. On Qwen3-7B-A1.5B, BF16 reaches an average score of 56.93, while vanilla NVFP4, NVFP4-Hadamard, and Averis obtain 55.56, 55.45, and 56.22, corresponding to average gaps of 1.37, 1.48, and 0.71 points, respectively. Averis thus gives the strongest 7B downstream result among the reported low-bit methods.

\textbf{Complementarity with Hadamard Transformation.}
Averis is orthogonal to Hadamard-based outlier smoothing. On Qwen3-0.6B, combining Averis with Hadamard transformation further reduces the loss gap from 1.19\% to 0.94\% and gives the best downstream average among FP4 variants. This suggests that Averis removes the dominant coherent mean-induced dynamic range, while Hadamard transformation can still provide additional element-space smoothing on the residual signal.

\begin{table*}[h]
\caption{Training loss and downstream performance comparison on Qwen3-0.6B and Qwen3-7B-A1.5B. Loss gap denotes the relative loss increase over BF16. Downstream results are reported in percentage points. BF16 is shown as the full-precision reference and separated from FP4 methods; bold numbers indicate the best result among FP4 methods within each model block.}
\label{tab:loss_downstream_combined}
\centering
\small
\setlength{\tabcolsep}{4pt}
\renewcommand{\arraystretch}{1.05}
\begin{tabular*}{\textwidth}{@{\extracolsep{\fill}}lccccccccc@{}}
\toprule
Method & Loss & Loss Gap & ArcC & ArcE & HellaSwag & LAMBADA & PIQA & RACE & Avg \\
\midrule
\multicolumn{10}{c}{\textbf{Qwen3-0.6B}} \\
\midrule
BF16
& 2.780 & --
& 28.92 & 53.37 & 49.83 & 47.76 & 71.00 & 55.31 & 51.03 \\
\cmidrule(lr){1-10}
NVFP4
& 2.855 & 2.70\%
& 27.73 & 52.48 & 48.15 & 44.93 & 69.80 & \textbf{55.55} & 49.77 \\
NVFP4-Hadamard
& 2.837 & 2.05\%
& 27.56 & 52.99 & \textbf{48.51} & 45.78 & \textbf{70.02} & 55.45 & 50.05 \\
Averis
& 2.813 & 1.19\%
& \textbf{29.44} & \textbf{53.66} & 48.15 & 44.44 & 69.64 & 55.49 & 50.14 \\
Averis-Hadamard
& \textbf{2.806} & \textbf{0.94\%}
& 28.92 & 52.86 & 48.43 & \textbf{46.96} & \textbf{70.02} & 54.60 & \textbf{50.30} \\
\midrule
\multicolumn{10}{c}{\textbf{Qwen3-7B-A1.5B}} \\
\midrule
BF16
& 2.824 & --
& 35.75 & 69.91 & 61.12 & 49.31 & 74.81 & 50.65 & 56.93 \\
\cmidrule(lr){1-10}
NVFP4
& 2.853 & 1.03\%
& 33.24 & 69.23 & 60.56 & 46.68 & 73.26 & 50.41 & 55.56 \\
NVFP4-Hadamard
& 2.855 & 1.10\%
& 33.52 & 69.32 & 60.22 & 46.58 & 72.34 & \textbf{50.71} & 55.45 \\
Averis
& \textbf{2.847} & \textbf{0.81\%}
& \textbf{34.88} & \textbf{69.70} & \textbf{60.75} & \textbf{48.58} & \textbf{74.13} & 49.27 & \textbf{56.22} \\
\bottomrule
\end{tabular*}
\end{table*}


\textbf{Runtime overhead comparison.}
We benchmark the preprocessing overhead of Averis against tiled Hadamard transformation. Tiled Hadamard reshapes $\mathbf{X}$ into $[l,m/16,16]$ and applies a $16\times16$ Hadamard transform along the last dimension, whereas Averis only computes the feature-wise mean of $\mathbf{X}$. As shown in Table~\ref{tab:averis_hadamard_overhead}, Averis is consistently faster on large activation shapes, achieving mean-latency speedups of $4.47\times$ and $4.72\times$ for $(l,m)=(512 \times 2048,4096)$ and $(512 \times 2048,8192)$, respectively. This advantage increases with activation size.


\begin{table*}[t]
\centering
\caption{Runtime overhead comparison between tiled Hadamard preprocessing and Averis mean extraction on large activation shapes. Latencies are reported in milliseconds. Speedup is computed from mean latency as $T_{\mathrm{Hadamard}}/T_{\mathrm{Averis}}$.}
\label{tab:averis_hadamard_overhead}
\small
\setlength{\tabcolsep}{5pt}
\begin{tabular*}{\textwidth}{@{\extracolsep{\fill}}llccc@{}}
\toprule
Shape $(l,m)$ & Method & Mean & Std & Speedup \\
\midrule
$(512 \times 2048,4096)$ & Tiled Hadamard & 9.1614 & 0.0128 & -- \\
$(512 \times 2048,4096)$ & Averis & 2.0494 & 0.0021 & $4.47\times$ \\
\midrule
$(512 \times 2048,8192)$ & Tiled Hadamard & 18.8421 & 0.0103 & -- \\
$(512 \times 2048,8192)$ & Averis & 3.9927 & 0.0028 & $4.72\times$ \\
\bottomrule
\end{tabular*}
\end{table*}

\textbf{End-to-end overhead on Blackwell.}
We further measure end-to-end training-step latency on a Blackwell GPU for both Qwen3-0.6B and Qwen3-7B-A1.5B. Table~\ref{tab:blackwell_overhead} reports the corresponding step times for vanilla NVFP4,{Averis}, and NVFP4-Hadamard. On Qwen3-0.6B, {Averis} increases end-to-end latency by only \(2.01\%\) relative to vanilla NVFP4, with an overhead that is \(29.6\%\) of the overhead of NVFP4-Hadamard. On Qwen3-7B-A1.5B, {Averis} incurs \(2.20\%\) end-to-end overhead, which is \(28.9\%\) of the overhead of tiled Hadamard. These results show that the low arithmetic overhead of {Averis} carries over to full-model training on modern Blackwell hardware across both model scales.

\begin{table}[h]
\centering
\caption{End-to-end training-step latency on Blackwell GPUs for Qwen3-0.6B and Qwen3-7B-A1.5B.}
\label{tab:blackwell_overhead}
\small
\begin{tabular}{lcccc}
\toprule
Setting & Qwen3-0.6B (ms) & 0.6B Overhead & Qwen3-7B-A1.5B (ms) & 7B Overhead \\
\midrule
NVFP4 & 5545.30 & -- & 24852.04 & -- \\
Averis & 5656.80 & 2.01\% & 25398.78 & 2.20\% \\
NVFP4-Hadamard & 5922.38 & 6.80\% & 26744.92 & 7.62\% \\
\bottomrule
\end{tabular}
\end{table}

\section{Related Work}

\subsection{Outlier Suppression and Low-Bit Quantization}

Outlier dimensions and extreme activation magnitudes have been a central
challenge in low-bit quantization of LLMs.
SmoothQuant~\cite{xiao2023smoothquant} redistributes activation and weight
scales to reduce outlier concentration.
GPTQ~\cite{frantar2022gptq}, AWQ~\cite{lin2024awq}, QLoRA~\cite{dettmers2023qlora},
and SVDQuant~\cite{li2024svdquant}
propose various strategies to absorb or compensate for large-magnitude
components.
QuaRot~\cite{ashkboos2024quarot} and HALO~\cite{ashkboos2025halo}
employ orthogonal rotations or Hadamard transforms to mitigate outliers.
Outlier Suppression~\cite{wei2023outlier} explicitly targets extreme
activation dimensions.
Recent work on FP8 and FP4 training further highlights the sensitivity of
low-precision regimes to activation dynamic range
\cite{micikevicius2022fp8,peng2023fp8,perez2023fp8,cao2025metis}.

\section{Conclusion}
\label{sec:conclusion}

FP4 training is attractive for efficient large language model training, but remains fragile because activation outliers inflate blockwise quantization scales and compress long-tail signals. This work identifies a coherent rank-one mean bias as a structural source of these outliers. The feature-wise mean aligns with the leading anisotropic spectral component, grows training time and increasingly dominates extreme activation magnitudes.

Motivated by this observation, we propose \emph{Averis}, a mean--residual splitting quantization method that isolates the coherent mean component before FP4 quantization of activations and output gradients. Averis directly reduces bias-induced dynamic-range inflation using only mean reductions and elementwise subtractions. Experiments on dense and MoE Qwen3 models show that Averis narrows the BF16 loss and downstream accuracy gaps, complements Hadamard transformation, and incurs only 2.2\% end-to-end overhead on Blackwell GPUs.

\textbf{Limitations.} First, due to computational constraints, our long-run training experiments are performed with FP4 simulation rather than directly executed end-to-end on Blackwell GPUs. Second, broader validation across more model scales, data distributions, and architectural variants remains future work. Third, we observe that the mean-bias phenomenon is less pronounced in output gradients than in activations; we therefore do not provide a dedicated ablation or specialized design for gradient-output mean bias, leaving more adaptive backward-pass treatments for future exploration.

\newpage


\newpage

\appendix

\section{Additional Mean-Bias Visualizations Across Layers and Training Stages}
\label{app:mean_bias_more_layers}

To verify that the pattern shown in Figure~\ref{fig:mean_bias_definition} is
not an isolated late-stage layer-27 artifact, we examine the same three-panel
diagnostic across additional layers and checkpoints. Figures
~\ref{fig:mean_bias_appendix_layer0_10k},
\ref{fig:mean_bias_appendix_layer0_170k}, and
\ref{fig:mean_bias_appendix_layer27_170k} report three representative cases:
layer 0 at 10k, layer 0 at 170k, and layer 27 at 170k. In all three cases, the
same qualitative structure persists: the singular spectrum exhibits a leading
spike, token-wise projections onto the mean direction remain predominantly
one-sided, and the mean vector is most strongly aligned with the top right
singular vector.

These supplementary examples highlight two aspects of generality. Across
training, the phenomenon is already visible in early checkpoints and becomes
stronger at later stages, as seen by comparing
Figures~\ref{fig:mean_bias_appendix_layer0_10k} and
\ref{fig:mean_bias_appendix_layer0_170k}. Across depth, the effect appears in
both shallow and deep layers, while becoming especially sharp in the deep
late-stage activation shown in Figure~\ref{fig:mean_bias_appendix_layer27_170k}.
Taken together, these observations show that mean bias is not a single-layer
artifact but a robust structural feature of LLM activations.

\begin{figure}[h]
\centering
\includegraphics[width=\linewidth]{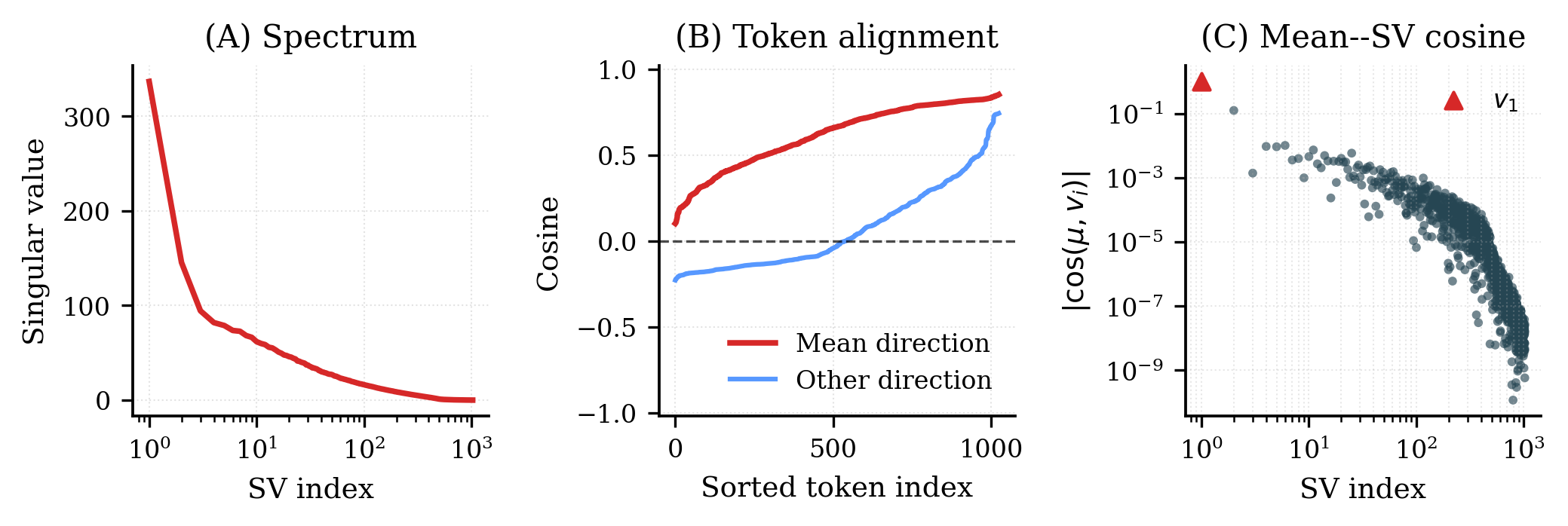}
\caption{Additional mean-bias visualization for layer 0 at training step 10k.
The leading spectral spike, predominantly one-sided token alignment to the mean
direction, and strong mean--top-singular-vector alignment are already visible at
an early checkpoint.}
\label{fig:mean_bias_appendix_layer0_10k}
\end{figure}

\begin{figure}[h]
\centering
\includegraphics[width=\linewidth]{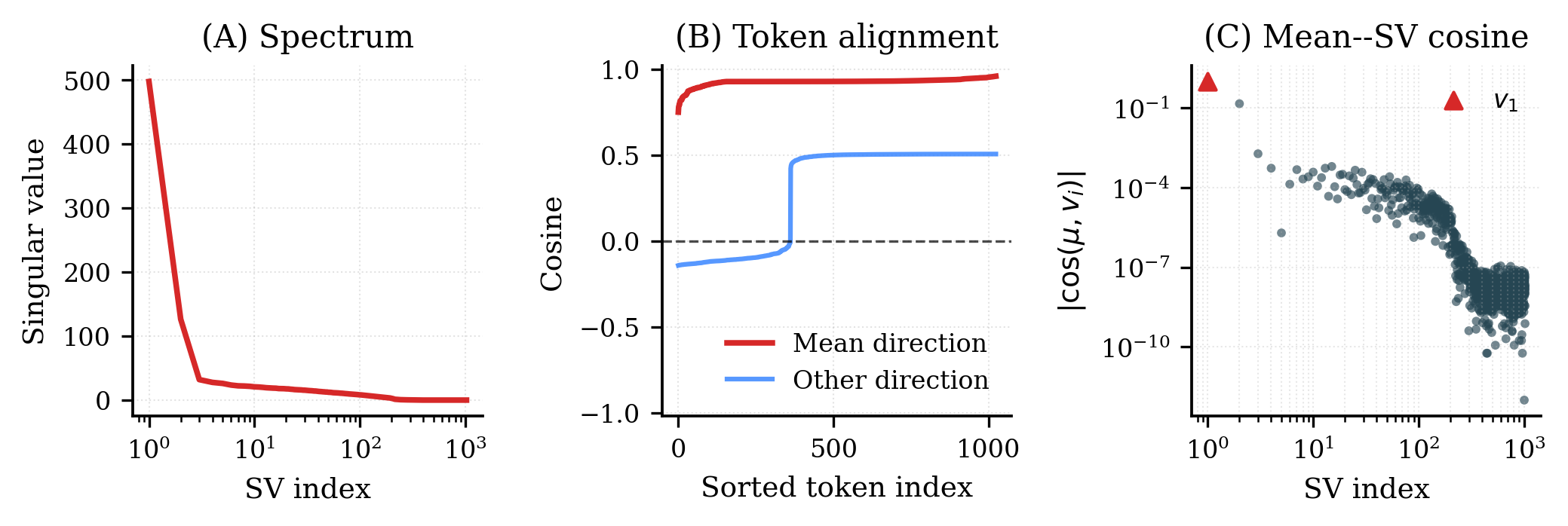}
\caption{Additional mean-bias visualization for layer 0 at training step 170k.
Compared with the early checkpoint, the same qualitative structure persists and
the coherence of the mean component becomes stronger.}
\label{fig:mean_bias_appendix_layer0_170k}
\end{figure}

\begin{figure}[h]
\centering
\includegraphics[width=\linewidth]{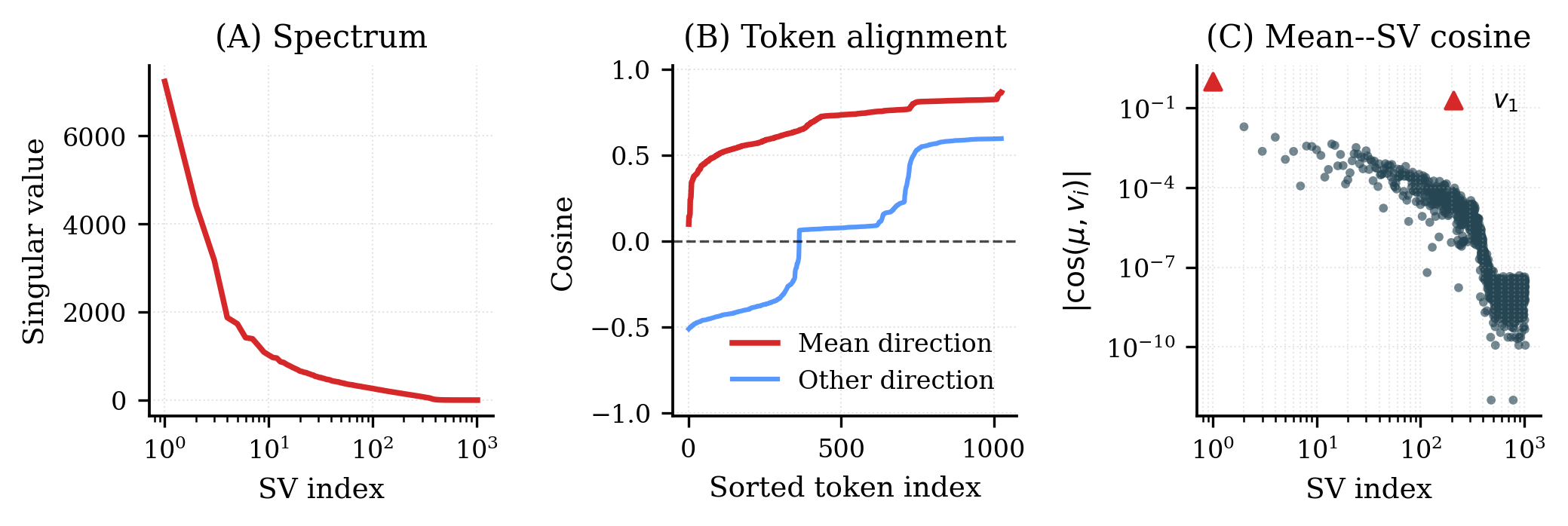}
\caption{Additional mean-bias visualization for layer 27 at training step 170k.
The phenomenon is especially pronounced in this deep late-stage activation,
matching the strongest regime discussed in the main text.}
\label{fig:mean_bias_appendix_layer27_170k}
\end{figure}

\section{Empirical Validation of the Diagonal Variance Approximation}
\label{app:variance_decomp}

We empirically validate Assumption~2 by comparing the row-sampling residual
variance of each column with its diagonal spectral approximation
$\tau_{j,\mathrm{diag}}^2$. Figure~\ref{fig:variance_decomp} shows that the
diagonal estimate closely tracks the empirical residual variance, while the
relative contribution of cross-mode covariance terms remains small. Quantitatively,
the cross-term contribution has median $0.006$ and $95$-th percentile $0.036$,
supporting the diagonal approximation used in the main-text tail model.

\begin{figure}[h]
\centering
\includegraphics[width=\linewidth]{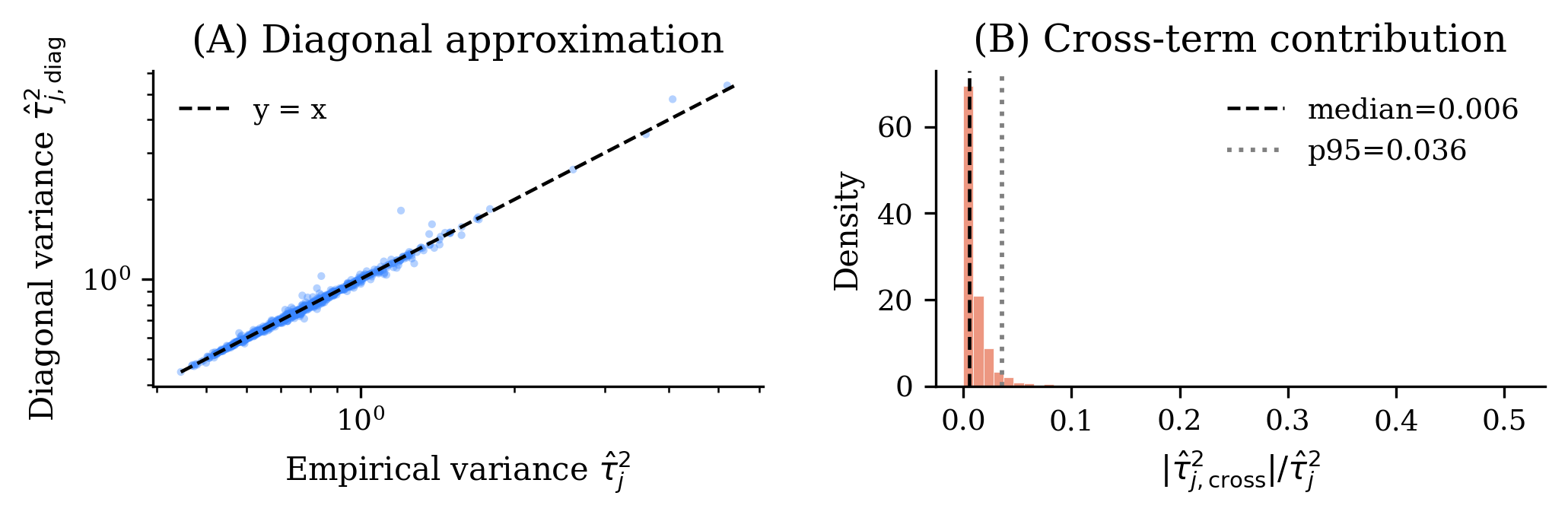}
\caption{Empirical validation of the diagonal variance approximation. The
diagonal spectral estimate closely matches the empirical residual variance, and
cross-mode covariance terms contribute little to the total variance.}
\label{fig:variance_decomp}
\end{figure}

\section{Residual Tail Contraction After Mean Removal}
\label{app:residual_outlier_distribution}

To support the tail-amplification mechanism in Theorem~\ref{thm:mean_bias_tail},
we directly compare the empirical value distributions of raw activations and
their mean-centered residuals. Figures~\ref{fig:raw_residual_layer0_170k} and
\ref{fig:raw_residual_layer27_170k} show representative late-stage examples
from a shallow layer (layer 0) and a deep layer (layer 27), both at the 170k
checkpoint. In both cases, subtracting the feature-wise mean visibly contracts
the high-magnitude tail of the distribution: the residual histogram is more
concentrated near zero and exhibits substantially fewer extreme values than the
raw activation histogram.

This pattern is consistent with the main-text interpretation. A coherent mean
shift contributes directly to extreme entrywise magnitudes by translating the
columnwise distribution away from zero. Once this mean component is removed, the
remaining residual fluctuations are much less heavy-tailed, especially in the
far-tail regime most relevant to blockwise quantization scales. These examples
therefore provide direct empirical support for the claim that mean removal
reduces the outlier population that drives FP4 dynamic-range inflation.

\begin{figure*}[h]
\centering
\begin{subfigure}{0.49\textwidth}
\centering
\includegraphics[width=\linewidth]{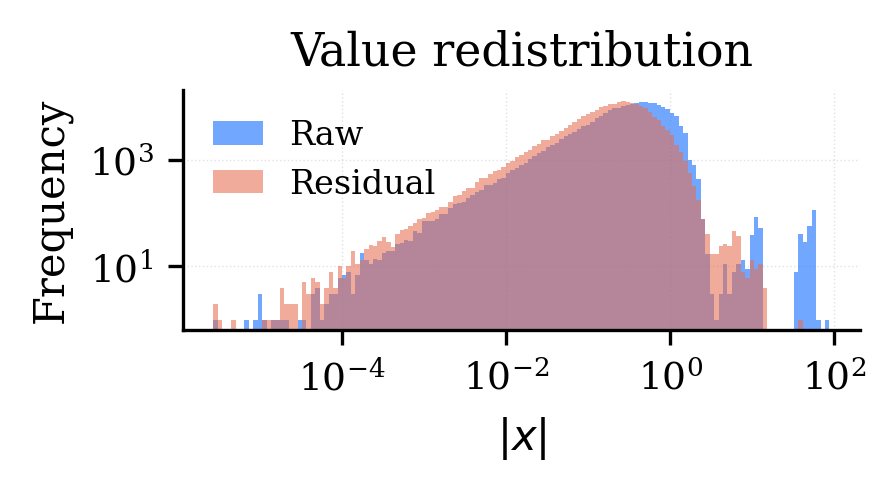}
\caption{Layer 0 at training step 170k. After subtracting the feature-wise
mean, the residual distribution is more concentrated and exhibits a
substantially weaker high-magnitude tail.}
\label{fig:raw_residual_layer0_170k}
\end{subfigure}\hfill
\begin{subfigure}{0.49\textwidth}
\centering
\includegraphics[width=\linewidth]{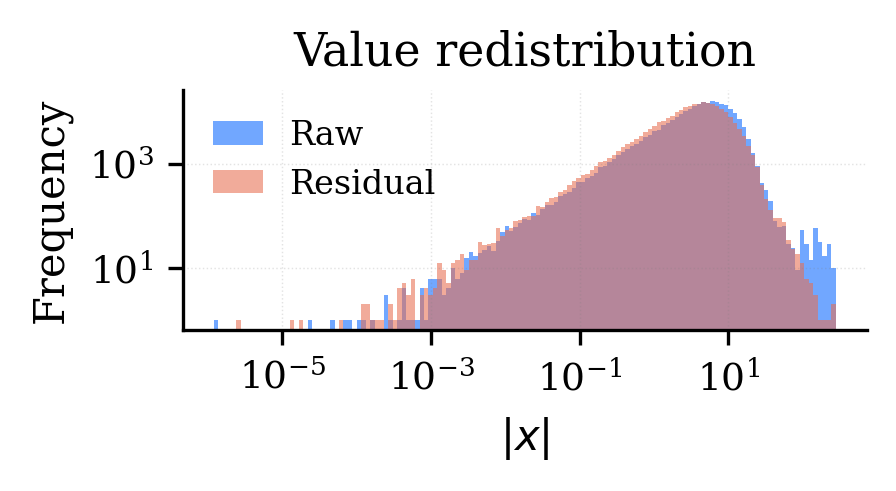}
\caption{Layer 27 at training step 170k. The contraction of the tail after
mean removal is even more pronounced in this deep late-stage activation, where
mean bias is strongest.}
\label{fig:raw_residual_layer27_170k}
\end{subfigure}
\caption{Raw-versus-residual value distributions at the 170k checkpoint for a
shallow layer and a deep layer. In both cases, subtracting the feature-wise
mean contracts the high-magnitude tail, with the effect becoming especially
pronounced in the deep late-stage activation.}
\end{figure*}

\section{Output-Gradient Mean Centering}
\label{app:gradout_mean_centering}

Although output gradients exhibit a much weaker mean-bias structure than
activations, mean centering still provides a small but consistent quantization
benefit. Figure~\ref{fig:gradout_three_panel} shows a representative three-panel
diagnostic for an output-gradient matrix. Compared with the activation examples
in the main text, the leading spectral spike is weaker, token-wise alignment to
the mean direction is less one-sided, and the mean vector is less sharply
aligned with the top right singular vector. In this sense, mean bias is not as
pronounced for output gradients as it is for activations.

Nevertheless, quantization results still show a slight benefit from mean
centering. Under NVFP4 quantization, the relative quantization error decreases
from $13.6\%$ to $13.5\%$ after centering the output-gradient matrix. This gain
is modest, but it is directionally consistent with the Averis design: even when
the coherent mean component is weak, isolating it can still slightly improve
low-bit numerical fidelity.

\begin{figure}[h]
\centering
\includegraphics[width=\linewidth]{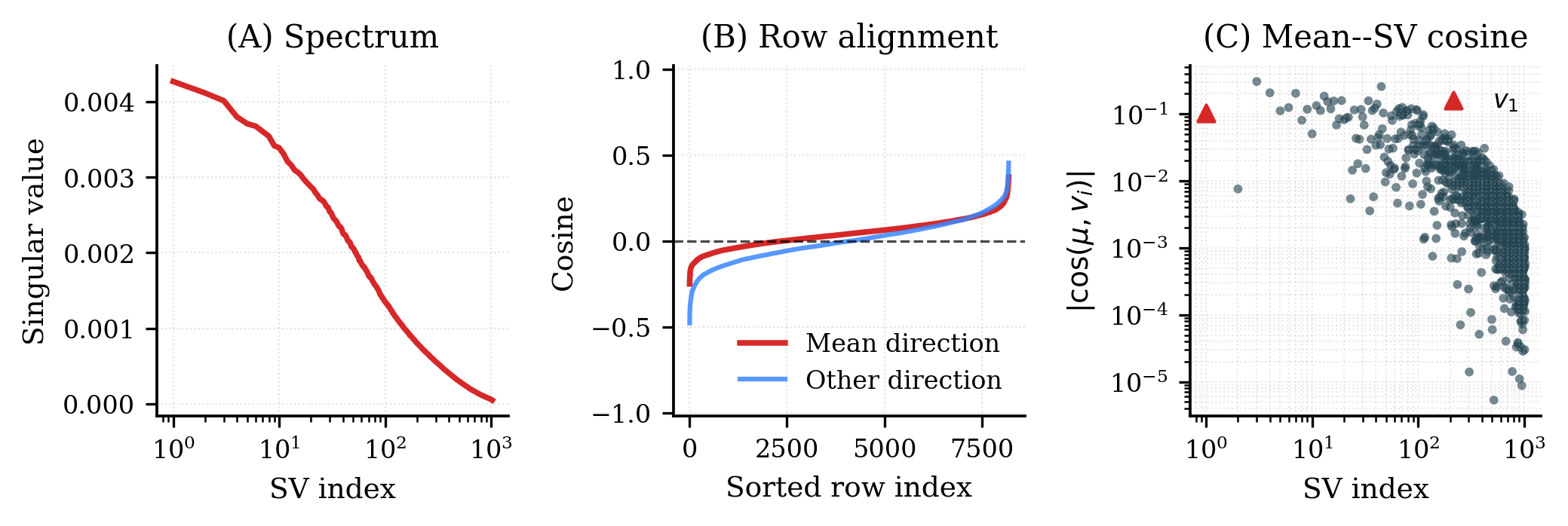}
\caption{Three-panel visualization of a representative output-gradient matrix.
Compared with activations, the mean-bias structure is weaker: spectral
dominance is milder, token-wise alignment to the mean direction is less
one-sided, and the mean vector is less sharply aligned with the top right
singular vector. Even so, mean centering still yields a slight reduction in
NVFP4 quantization error.}
\label{fig:gradout_three_panel}
\end{figure}

\section{Proof of Theorem~\ref{thm:mean_bias_tail}}
\label{appendix:proof_mean_bias_tail}

\begin{proof}
We prove the theorem in several steps. Throughout the proof, probabilities are
understood under the marginal row-coordinate model stated in the theorem,
conditional on the spectral quantities entering $m_j$ and $\tau_j$. For notational
simplicity, we keep the conditioning implicit when no ambiguity arises.

\paragraph{Step 1: Exact spectral expression for the column mean.}
By the singular value decomposition of $\mathbf X$,
\[
X_{ij}
=
\sum_{k=1}^{r}\sigma_k u_{ki}v_{kj}.
\]
Averaging over the row index $i$ gives
\[
m_j
=
\frac{1}{l}\sum_{i=1}^{l}X_{ij}
=
\frac{1}{l}
\sum_{i=1}^{l}
\sum_{k=1}^{r}
\sigma_k u_{ki}v_{kj}.
\]
Interchanging the summations yields
\[
m_j
=
\sum_{k=1}^{r}
\sigma_k v_{kj}
\left(
\frac{1}{l}\sum_{i=1}^{l}u_{ki}
\right).
\]
Since
\[
\beta_k
=
\left\langle
\mathbf u_k,\frac{\mathbf 1}{\sqrt l}
\right\rangle
=
\frac{1}{\sqrt l}\sum_{i=1}^{l}u_{ki},
\]
we have
\[
\frac{1}{l}\sum_{i=1}^{l}u_{ki}
=
\frac{\beta_k}{\sqrt l}.
\]
Therefore,
\[
m_j
=
\frac{1}{\sqrt l}
\sum_{k=1}^{r}
\sigma_k\beta_k v_{kj},
\]
which is the column-mean expansion used in the theorem. The leading-mode
contribution is correspondingly
\[
m_j^{(1)}
=
\frac{\sigma_1\beta_1}{\sqrt l}v_{1j}.
\]

\paragraph{Step 2: Exact mean--fluctuation decomposition of the row-sampled entry.}
Let $I\sim\mathrm{Unif}(\{1,\dots,l\})$ and define
\[
Y_j=X_{Ij}.
\]
Using the SVD again,
\[
Y_j
=
\sum_{k=1}^{r}
\sigma_k u_{kI}v_{kj}.
\]
For each singular mode, add and subtract its row-average contribution
$\beta_k/\sqrt l$:
\[
Y_j
=
\sum_{k=1}^{r}
\sigma_k
\left(
\frac{\beta_k}{\sqrt l}
+
u_{kI}-\frac{\beta_k}{\sqrt l}
\right)
v_{kj}.
\]
Splitting the two terms gives
\[
Y_j
=
\underbrace{
\frac{1}{\sqrt l}
\sum_{k=1}^{r}
\sigma_k\beta_k v_{kj}
}_{m_j}
+
\underbrace{
\sum_{k=1}^{r}
\sigma_k
\left(
u_{kI}-\frac{\beta_k}{\sqrt l}
\right)
v_{kj}
}_{\eta_{Ij}}.
\]
Hence
\[
Y_j=m_j+\eta_{Ij}.
\]
Moreover, under the row-sampling distribution,
\[
\mathbb E_I[\eta_{Ij}]
=
\mathbb E_I[Y_j]-m_j
=
\frac{1}{l}\sum_{i=1}^{l}X_{ij}-m_j
=
0.
\]
Thus $\eta_{Ij}$ is a centered row-coordinate fluctuation.

\paragraph{Step 3: Residual variance scale and marginal Gaussian approximation.}
The theorem assumes a marginal Gaussian approximation for the row-sampled
fluctuation:
\[
\eta_{Ij}\,\dot\sim\,N(0,\tau_j^2).
\]
This is a marginal assumption on the randomly sampled coordinate $\eta_{Ij}$.
It does not require the row residuals
$\{\eta_{ij}\}_{i=1}^{l}$ to be mutually independent.

The variance scale $\tau_j^2$ is obtained by keeping the diagonal mode-wise
variance contributions. For a single singular mode,
\[
\mathbb E_I
\left[
\left(
u_{kI}-\frac{\beta_k}{\sqrt l}
\right)^2
\right]
=
\frac{1}{l}\sum_{i=1}^{l}u_{ki}^2
-
\left(
\frac{1}{l}\sum_{i=1}^{l}u_{ki}
\right)^2.
\]
Since $\|\mathbf u_k\|_2=1$ and
\[
\frac{1}{l}\sum_{i=1}^{l}u_{ki}
=
\frac{\beta_k}{\sqrt l},
\]
we obtain
\[
\mathbb E_I
\left[
\left(
u_{kI}-\frac{\beta_k}{\sqrt l}
\right)^2
\right]
=
\frac{1}{l}
-
\frac{\beta_k^2}{l}
=
\frac{1-\beta_k^2}{l}.
\]
Therefore, the diagonal contribution of mode $k$ to the row-sampling variance of
$\eta_{Ij}$ is
\[
\frac{1}{l}
\sigma_k^2(1-\beta_k^2)v_{kj}^2.
\]
Neglecting cross-mode covariance terms gives the variance approximation
\[
\tau_j^2
=
\frac{1}{l}
\sum_{k=1}^{r}
\sigma_k^2(1-\beta_k^2)v_{kj}^2.
\]
Under the marginal Gaussian row-coordinate approximation,
\[
Y_j=m_j+\eta_{Ij}
\,\dot\sim\,
N(m_j,\tau_j^2).
\]

\paragraph{Step 4: Two-sided Gaussian tail formula.}
Let
\[
Z_j
=
\frac{Y_j-m_j}{\tau_j}.
\]
Under the marginal Gaussian approximation, $Z_j\dot\sim N(0,1)$. For any
threshold $t>0$,
\[
\mathbb P(|Y_j|>t)
=
\mathbb P(Y_j>t)
+
\mathbb P(Y_j<-t).
\]
The first term is
\[
\mathbb P(Y_j>t)
=
\mathbb P\left(
Z_j>\frac{t-m_j}{\tau_j}
\right)
=
Q\left(
\frac{t-m_j}{\tau_j}
\right).
\]
The second term is
\[
\mathbb P(Y_j<-t)
=
\mathbb P\left(
Z_j<\frac{-t-m_j}{\tau_j}
\right)
=
\Phi\left(
\frac{-t-m_j}{\tau_j}
\right)
=
Q\left(
\frac{t+m_j}{\tau_j}
\right).
\]
Since the two-sided tail depends on the sign of $m_j$ only through $|m_j|$, we
can write
\[
\mathbb P(|Y_j|>t)
\approx
Q\left(
\frac{t-|m_j|}{\tau_j}
\right)
+
Q\left(
\frac{t+|m_j|}{\tau_j}
\right).
\]
This proves~\eqref{eq:tail_formula}.

\paragraph{Step 5: Far-tail one-sided dominance.}
Let
\[
s_j\triangleq |m_j|,
\qquad
a_j\triangleq \frac{t-s_j}{\tau_j},
\qquad
b_j\triangleq \frac{t+s_j}{\tau_j}.
\]
Under the far-tail conditions in~\eqref{eq:far_tail_conditions}, we have
\[
a_j\to+\infty,
\qquad
\frac{ts_j}{\tau_j^2}\to\infty.
\]
Mills' ratio gives
\[
Q(x)
\sim
\frac{1}{x\sqrt{2\pi}}e^{-x^2/2},
\qquad x\to+\infty.
\]
Therefore,
\[
\frac{Q(b_j)}{Q(a_j)}
\sim
\frac{a_j}{b_j}
\exp\left(
-\frac{b_j^2-a_j^2}{2}
\right).
\]
Since
\[
b_j^2-a_j^2
=
\frac{(t+s_j)^2-(t-s_j)^2}{\tau_j^2}
=
\frac{4ts_j}{\tau_j^2},
\]
we have
\[
\frac{Q(b_j)}{Q(a_j)}
\sim
\frac{t-s_j}{t+s_j}
\exp\left(
-\frac{2ts_j}{\tau_j^2}
\right).
\]
The condition $ts_j/\tau_j^2\to\infty$ implies
\[
\frac{Q(b_j)}{Q(a_j)}\to 0.
\]
Hence
\[
Q\left(
\frac{t+|m_j|}{\tau_j}
\right)
=
o\left(
Q\left(
\frac{t-|m_j|}{\tau_j}
\right)
\right).
\]
Combining this with the two-sided tail formula gives
\[
\mathbb P(|Y_j|>t)
\sim
Q\left(
\frac{t-|m_j|}{\tau_j}
\right),
\]
which proves~\eqref{eq:tail_asymptotic}.

\paragraph{Step 6: Tail amplification relative to the variance-only baseline.}
Consider the variance-only baseline
\[
Y_j^{(0)}\sim N(0,\tau_j^2).
\]
Its two-sided tail probability is
\[
\mathbb P(|Y_j^{(0)}|>t)
=
2Q(t/\tau_j).
\]
Using the far-tail approximation from Step 5,
\[
\frac{
\mathbb P(|Y_j|>t)
}{
2Q(t/\tau_j)
}
\approx
\frac{
Q((t-|m_j|)/\tau_j)
}{
2Q(t/\tau_j)
}.
\]
Applying Mills' ratio to the numerator gives
\[
Q\left(
\frac{t-|m_j|}{\tau_j}
\right)
\approx
\frac{\tau_j}{(t-|m_j|)\sqrt{2\pi}}
\exp\left(
-\frac{(t-|m_j|)^2}{2\tau_j^2}
\right),
\]
while applying it to the baseline tail gives
\[
2Q\left(
\frac{t}{\tau_j}
\right)
\approx
\frac{2\tau_j}{t\sqrt{2\pi}}
\exp\left(
-\frac{t^2}{2\tau_j^2}
\right).
\]
Taking the ratio yields
\[
\frac{
\mathbb P(|Y_j|>t)
}{
2Q(t/\tau_j)
}
\approx
\frac{t}{2(t-|m_j|)}
\exp\left(
\frac{
t^2-(t-|m_j|)^2
}{
2\tau_j^2}
\right).
\]
Since
\[
t^2-(t-|m_j|)^2
=
2t|m_j|-m_j^2,
\]
we obtain
\[
\frac{
\mathbb P(|Y_j|>t)
}{
2Q(t/\tau_j)
}
\approx
\frac{t}{2(t-|m_j|)}
\exp\left(
\frac{2t|m_j|-m_j^2}{2\tau_j^2}
\right).
\]
This proves~\eqref{eq:tail_ratio}.

\paragraph{Step 7: From mean dominance to mean-driven tail amplification.}
Assumption~\textnormal{(A2)} states that
\[
\frac{|m_j|}{\tau_j}\gg 1.
\]
In the far-tail regime $t>|m_j|$, the exponent in~\eqref{eq:tail_ratio}
satisfies
\[
\frac{2t|m_j|-m_j^2}{2\tau_j^2}
=
\frac{|m_j|(2t-|m_j|)}{2\tau_j^2}
>
\frac{m_j^2}{2\tau_j^2}.
\]
Thus, when $|m_j|/\tau_j\gg 1$, the exponential factor in~\eqref{eq:tail_ratio}
is large. Consequently, the high-threshold exceedance probability in column $j$
is exponentially amplified by the coherent mean shift relative to a zero-mean
variance-only fluctuation with the same variance scale.

\paragraph{Step 8: Attribution to the aligned leading singular mode.}
Assumption~\textnormal{(A3)} states that
\[
|m_j-m_j^{(1)}|\ll |m_j^{(1)}|.
\]
Therefore,
\[
m_j
=
m_j^{(1)}+o(|m_j^{(1)}|)
=
\frac{\sigma_1\beta_1}{\sqrt l}v_{1j}
+
o\left(
\left|
\frac{\sigma_1\beta_1}{\sqrt l}v_{1j}
\right|
\right).
\]
Hence the coherent mean shift responsible for the tail amplification is
asymptotically the contribution of the aligned leading singular mode. Across all
rows, this leading mean contribution corresponds to the rank-one matrix
\[
\frac{\sigma_1\beta_1}{\sqrt l}\,
\mathbf 1\mathbf v_1^\top.
\]
Therefore, under the aligned-spike dominance condition, the mean-driven
far-tail amplification is controlled specifically by the rank-one mean-bias
component. This completes the proof.
\end{proof}

\newpage
\section*{NeurIPS Paper Checklist}

\begin{enumerate}

\item {\bf Claims}
\item[] Question: Do the main claims made in the abstract and introduction accurately reflect the paper's contributions and scope?
\item[] Answer: \answerYes{}
\item[] Justification: The abstract and introduction clearly state the core claims: identifying mean bias as the dominant source of activation outliers, providing theoretical analysis of its amplification effect (Section 2.3), and proposing the Averis method (Section 3). These claims are supported by empirical validation (Section 4).

\item {\bf Limitations}
    \item[] Question: Does the paper discuss the limitations of the work performed by the authors?
    \item[] Answer: \answerYes{}.
    \item[] Justification: We discuss the limitations in the dedicated limitations paragraph. In particular, we note that our long-run experiments rely on FP4 simulation rather than full end-to-end Blackwell execution, that broader validation across model scales, data distributions, and architectures remains future work, and that gradient-output mean bias is less pronounced than activation mean bias and is therefore not separately specialized in this work.

\item {\bf Theory assumptions and proofs}
\item[] Answer: \answerYes{}
\item[] Justification: All theoretical results are presented with explicit assumptions (Section 2.3, Assumptions 1–3), and complete proofs are provided in Appendix E. The derivations are clearly structured and cross-referenced.

\item {\bf Experimental result reproducibility}
\item[] Answer: \answerYes{}
\item[] Justification: The paper provides sufficient details to reproduce experiments, including model configurations (Qwen3-0.6B and Qwen3-7B-A1.5B), dataset (DCLM), training token budgets, quantization setup (W4A4G4 NVFP4), and evaluation protocols (Section 4).

\item {\bf Open access to data and code}
\item[] Answer: \answerYes{}
\item[] Justification: An anonymized code repository is provided at submission time (see abstract), allowing reviewers to reproduce the results while preserving anonymity.

\item {\bf Experimental setting/details}
\item[] Answer: \answerYes{}
\item[] Justification: The paper specifies all relevant training and evaluation details, including datasets, model architectures, quantization schemes, baselines, and evaluation metrics (Section 4).

\item {\bf Experiment statistical significance}
\item[] Answer: \answerNo{}
\item[] Justification: The training results are based on single training runs without reporting error bars, due to the high computational cost of large-scale LLM training.

\item {\bf Experiments compute resources}
\item[] Answer: \answerYes{}
\item[] Justification: The paper specifies the hardware platforms used (Hopper GPUs for FP4 simulation and Blackwell GPUs for runtime evaluation) and reports end-to-end latency and overhead measurements (Section 4.1).

\item {\bf Code of ethics}
\item[] Answer: \answerYes{}
\item[] Justification: The research complies with the NeurIPS Code of Ethics and focuses on improving training efficiency and stability without involving sensitive data or human subjects.

\item {\bf Broader impacts}
\item[] Answer: \answerYes{}
\item[] Justification: The proposed method improves the efficiency and stability of LLM training, which can reduce computational cost and energy consumption. However, improving training efficiency may also facilitate broader deployment of large models, potentially amplifying risks such as misuse in content generation.

\item {\bf Safeguards}
\item[] Answer: \answerNA{}
\item[] Justification: The paper does not release new high-risk models or datasets that require safeguards.

\item {\bf Licenses for existing assets}
\item[] Answer: \answerNo{}
\item[] Justification: While the paper uses publicly available datasets and models (e.g., DCLM and Qwen3), explicit license information is not discussed.

\item {\bf New assets}
\item[] Answer: \answerNo{}
\item[] Justification: The paper does not introduce new datasets or benchmarks; it proposes a method and provides implementation.

\item {\bf Crowdsourcing and research with human subjects}
\item[] Answer: \answerNA{}
\item[] Justification: The paper does not involve human subjects or crowdsourcing.

\item {\bf Institutional review board (IRB) approvals or equivalent}
\item[] Answer: \answerNA{}
\item[] Justification: The paper does not involve human subjects.

\item {\bf Declaration of LLM usage}
\item[] Answer: \answerNA{}
\item[] Justification: LLMs are the subject of study and not used as a methodological component in developing the proposed approach.

\end{enumerate}

\end{document}